\pdfoutput=1

\documentclass[11pt]{article}

\usepackage[table,xcdraw]{xcolor}
\usepackage{amsmath}
\usepackage{algorithm}
\usepackage{algpseudocode}
\usepackage{enumitem}
\usepackage{stfloats}
\usepackage{acl}

\usepackage{times}
\usepackage{latexsym}

\usepackage[T1]{fontenc}

\usepackage[utf8]{inputenc}

\usepackage{microtype}

\usepackage{hyperref}
\usepackage{url}
\usepackage{geometry}
\usepackage{booktabs}
\usepackage{tabularx}
\usepackage{graphicx} 
\usepackage{pifont}

\usepackage{graphicx}
\usepackage[table]{xcolor}
\usepackage{subcaption}
\usepackage{booktabs}
\usepackage{amsmath}
\usepackage{amssymb}

\usepackage{mathtools}

\title{
Mitigating Degree Bias Adaptively with \\ Hard-to-Learn Nodes in
Graph Contrastive Learning}

\author{Jingyu Hu \\ University of Bristol \\ ym21669@bristol.ac.uk
        \And  Hongbo Bo \\ University of Bristol \\ hongbo.bo@bristol.ac.uk \And
        Jun Hong \\ University of the West of England \\ jun.hong@uwe.ac.uk
        \AND
        Xiaowei Liu \\ University of Bristol \\ xiaowei.liu@bristol.ac.uk
        \And
        Weiru Liu \\ University of Bristol \\ weiru.liu@bristol.ac.uk
        }

\begin{document}
\maketitle
\begin{abstract}
Graph Neural Networks (GNNs) often suffer from degree bias in node classification tasks, where prediction performance varies across nodes with different degrees. Several approaches, which adopt Graph Contrastive Learning (GCL), have been proposed to mitigate this bias. However, the limited number of positive pairs and the equal weighting of all positives and negatives in GCL still lead to low-degree nodes acquiring insufficient and noisy information.
This paper proposes the Hardness Adaptive Reweighted (HAR) contrastive loss to mitigate degree bias. It adds more positive pairs by leveraging node labels and adaptively weights positive and negative pairs based on their learning hardness. In addition, we develop an experimental framework named SHARP to extend HAR to a broader range of scenarios. Both our theoretical analysis and experiments validate the effectiveness of SHARP. The experimental results across four datasets show that SHARP achieves better performance against baselines at both global and degree levels. 
\noindent
\textbf{Keywords:} Degree Bias, Graph Contrastive Learning, Adaptive Learning, Node Classification

\end{abstract}

\section{Introduction}

One common approach to graph representation learning is to use a Graph Neural Network (GNN) like a Graph Convolutional Network (GCN) \cite{kipf2016semi} to learn graph embeddings. These learned embeddings have proven effective in many downstream tasks. However, most of these studies have focused mainly on global performance, while their subgroup-level performance has often been overlooked. Though their global performance is good, recent studies \cite{tang2020investigating,liu2023generalized} have shown that the predictions of GNNs can exhibit structural unfairness. In node classification tasks,
GNNs tend to perform better on high-degree nodes than low-degree ones \cite{liu2023generalized,wang2022uncovering}. This is also known as degree bias: the performance disparity across nodes with different degrees.

One of the most common observations on the origin of degree bias is that the neighborhoods of low-degree nodes provide insufficient and noisy information for generating their effective representations \cite{subramonian2024theoretical,tang2020investigating}. GNNs learn node representations by aggregating feature information from neighbourhoods. High-degree nodes, with more neighbors, can collect a richer set of features than low-degree nodes. Furthermore, as the few neighbors of the low-degree node may come from different classes, the aggregated noisy information can even confuse the GNNs in classification tasks. Figure 1 (Left) shows a degree bias example of using a GNN to classify occupations, where GNNs can achieve worse performance for low-degree (painter, conductor) nodes compared to high-degree nodes (music producer). For instance, when GNNs classify the conductor node, they will aggregate features from its neighbors (two producer nodes); this noisy information can mislead the model into classifying it as a producer rather than a conductor.

\begin{figure*}[http]
    \centering
    \includegraphics[width=\linewidth]{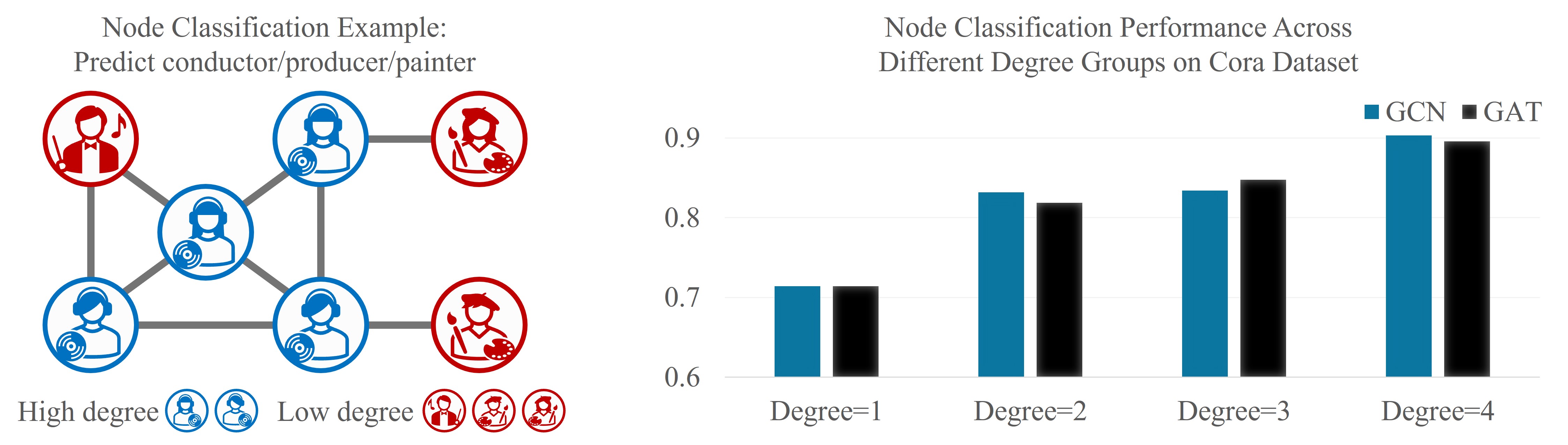}
    \caption{(Left.) The degree bias example in node classification: GNNs can perform worse on low-degree nodes (e.g., painter, conductor) than on high-degree nodes (e.g., producer) due to insufficient or noisy information from neighboring nodes for low-degree nodes. (Right.) Empirical results of degree bias. The different prediction performances across different degree groups on Cora: Low-degree nodes tend to have lower utility performance compared to high-degree nodes. }
    \label{fig:enter-label}
\end{figure*}

Recent studies~\cite{wang2022uncovering,kose2022fair,ling2023learning} have found that applying Graph Contrastive Learning (GCL) can help achieve fair representation learning.
GCL~\cite{zhu2020deep,zhu2021graph} combines GNNs with contrastive learning~\cite{chen2020simple,he2020momentum}. 
The typical GCL framework first generates two augmented views based on the transformation to the features and edges of the node, then defines a contrastive loss function to learn node representations from the generated graph views. GCL can mitigate degree bias to some extent by reducing the model's reliance on neighbourhoods by introducing augmented views.
The contrastive loss in GCL pulls positives closer while pushing negatives apart. 
For each node, the commonly adopted contrastive loss considers only the two augmented views as positives, and all other nodes as negatives. 

This assumption often leads to insufficient positive pair sampling for low-degree nodes which disadvantages them. 
One positive pair is insufficient to adequately support low-degree nodes, and degree bias still persists.
Supervised contrastive learning loss (SCL) \cite{khosla2020supervised} provides one solution by introducing more positive pairs for low-degree nodes. 
SCL treats all nodes of the same class as positives and all the other nodes as negatives, and assigns equal weights to every node.
However, as the example shown in Figure 1, different nodes have different levels of learning hardness:
Some are more informative and help form a clearer embedding space, whereas the others can contain noisier information to adversarially affect the latent space structure learning. Assigning equal weights to noisy nodes may prevent the model from capturing more important information.

In this paper, we propose a new contrastive loss that (1) includes more positives and (2) assigns different weights to positives and negatives to focus more on those with important information and less on the others with noisy information, thereby mitigating degree bias. We first apply SCL to ensure sufficient positives. Then, we design \textit{learning hardness} to quantify how hard each node is to learn.
For each node, both its augmented nodes and nodes from different classes with which the node shares high similarity are treated as hard-to-learn nodes.
We propose Hardness Adaptive Reweighted (HAR) contrastive loss that adaptively assigns higher weights to hard-to-learn nodes and lower weights to easier ones, ensuring the model focuses on the hard-to-learn cases.
Providing additional and more accurate information for low-degree nodes, which originally has insufficient information, helps accelerate their representation learning and reduce degree bias.

In addition, we extend our HAR loss to the semi-supervised GCL scenario to validate its adaptivity in real-world applications where not all samples are labelled. Specifically, we separate the training data into two subsets—the labelled and unlabelled parts—to simulate the scenario where part of the training data is unlabelled. 
Our proposed \underline{S}emi-supervised \underline{H}ardness \underline{A}daptive \underline{R}eweighted with \underline{P}seudo-labels (SHARP) consists of two steps. First, the model is pre-trained using the labelled training data. Second, the pre-trained model generates pseudo-labels for the unlabelled training data. We then use the pseudo-labelled training data to further fine-tune the model. Both pre-training and fine-tuning follow the GCL workflow. 
Overall, the main contributions of the paper are as follows:
    
\begin{itemize}
    \item \textbf{Hardness Adaptive Reweighted (HAR)} contrastive loss is proposed to address the degree bias. It provides more accurate positives and negatives for low-degree nodes, and assigns them adaptive weights. It enables the model to focus more on important information and less on noisy information. Our results show that low-degree nodes benefit more from adaptive weighting.
    \item \textbf{Empirical and Theoretical Analysis} We extend the application of HAR to the semi-supervised setting and propose the SHARP framework to validate HAR's adaptivity in broader scenarios. Our theoretical analysis also proves that (1) our proposed HAR loss is bounded and (2) even under the worst-case hyperparameter settings, the misclassification risk of SHARP also remains bounded.
    \item \textbf{Extensive Experiments} are performed across four datasets using two different GNN encoders, comparing our approach with four baselines. 
    The results show SHARP's effective degree bias mitigation without adversally affecting the overall performance.
    Globally, our SHARP can achieve comparative prediction performance with fewer labelled samples (even with 50\% fewer). Visualization of the latent feature space also validates that SHARP can produce clearer class boundaries. At the node degree level, SHARP shows particular enhancement in low-degree groups. 
\end{itemize}   

\section{Related Work}
\subsection{Graph Representation Learning}
Graph representation learning embeds graph-structured data into a feature space. The learned graph representations are generally integrated with task-specific neural networks to perform downstream tasks at various levels: graph-level, link-level, or node-level. This paper focuses on node-level tasks, which aim to predict the label of each node in a graph. GNNs and GCL are two common ways of generating embeddings for node classification tasks.

\textbf{Graph Neural Networks (GNNs)} are proposed to extend neural networks to process graph-structured data which has nodes features and connections between nodes. The core idea of GNNs is message passing between nodes, updating feature values by considering connected nodes layer by layer. Graph layers (GLs) define how nodes send, receive, and update messages. Convolution-based (e.g., GCN \cite{kipf2016semi}), attention-based (e.g., GAT \cite{velivckovic2017graph}), and aggregation-based (e.g., GraphSAGE \cite{hamilton2017inductive}) are three common types of graph layers~\cite{bronstein2021geometric}.

\textbf{Graph Contrastive Learning (GCL)} combines GNNs with Contrastive Learning (CL) which enhances the robustness of learned representations. CL is a learning scheme that focuses on pulling similar/positive samples closer while pushing dissimilar/negative samples away using a contrastive loss. 
Despite the abundance of contrastive loss variations \cite{chen2020simple,assran2020supervision,khosla2020supervised}, most GCL methods still use the common contrastive loss: it treats two augmented node originated from the same node as the positive pair, while treating all others as equally important negatives. 
This can make it hard to handle degree bias as we will discuss in detail below.


\begin{figure}[t]
    \centering
    \includegraphics[width=\linewidth]{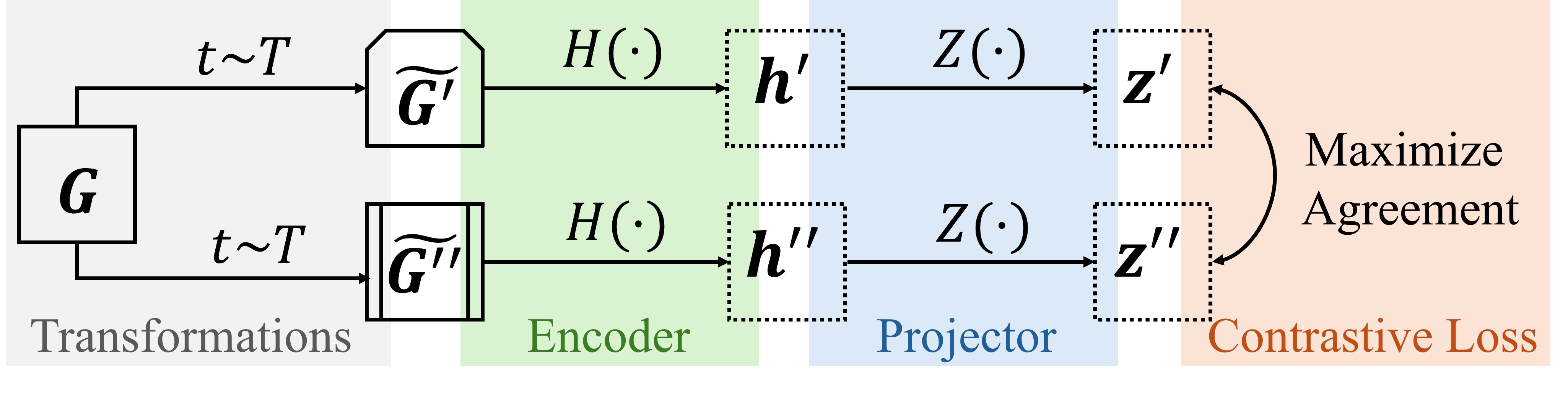}
    \caption{The Workflow of Contrastive Learning}
    \label{fig:workflow_simCLR_}
\end{figure}

\subsection{Degree Bias in Node Classification}

Based on graph connectivity, nodes can be split into low-degree and high-degree groups. \cite{zhao2024dahgn,subramonian2024theoretical} show that GNNs tend to have higher utility for high-degree nodes than low-degree nodes. This difference in prediction performance across nodes with different degrees is referred to as degree bias. \cite{subramonian2024theoretical} provides a detailed theoretical analysis and summarizes common hypotheses of degree bias origins, which the most common  hypotheses is that neighborhoods of low-degree nodes contain insufficient and noisy information \cite{feng2021should,xiao2021learning}.

\noindent
\textbf{Bias Mitigation Methods}
Many GNN-based degree bias mitigation methods \cite{yun2022lte4g,kang2021fair,liang2023tackling} have been proposed. For example, \cite{liu2021tail} introduced the concept of transferable neighborhood translation, which can transfer information from high-degree nodes to low-degree nodes to enhance their representations. \cite{han2024marginal} proposes a hop-aware attentive information aggregation scheme to expand the neighborhood for marginal nodes, thereby achieving structural fairness.

Recent work \cite{wang2022uncovering} found that GCL exhibits better structural fairness and less degree bias, compared with pure GNNs. 
GRACE \cite{zhu2020deep} applies GCL to node classification by generating contrastive views through the random masking of nodes and edges.
To further mitigate degree bias, many GCL-based mitigation methods have been proposed.  
GRADE \cite{wang2022uncovering} proposed a pre-defined graph augmentation pattern. Specifically, it follows the process of GRACE with a modified random edge-dropping scheme by deliberately dropping edges for high-degree nodes while adding edges for low-degree nodes. \cite{zhang2023understanding} applies random graph coarsening to data augmentation in GCL and mitigates degree bias by contrasting the coarsened graph with the original graph. \cite{kose2022fair} introduces adaptive graph augmentation to make GCL fairness-aware. Furthermore, \cite{zhao2024dahgn} performs GCL on heterogeneous graphs to mitigate degree bias.

These GCL mitigation methods generally tackle degree bias by introducing optimized data augmentation in the pre-processing stage, while still using the common contrastive loss. 
However, the common contrastive loss has two underlying limitations: (1) It considers only one positive pair per node. For low-degree nodes with few connections, this restricts their information gathering. (2) It assumes all positives and negatives contribute equally to the loss function. However, some of them can contain more important information and are harder to learn in practice.

To address these issues and mitigate degree bias, we propose a new contrastive loss that (1) is specifically designed to identify positive and negative pairs based on label information, allowing low-degree nodes to learn more information, and (2) assigns adaptive weights based on learning hardness to provide tailored information, enabling low-degree nodes to receive less noise and more focused learning.

\section{Methodology}
We begin with a brief introduction to the basics of graph contrastive learning. This is followed by a detailed description of our proposed Hardness Adaptive Reweighted (HAR) loss, which helps mitigate degree bias. Last, we describe an experimental pipeline named SHARP that extends HAR loss across broader scenarios for enhanced robustness.

\subsection{Preliminary}
The graph data is denoted as \(G = (V, E)\), where \(V\) is a set of $N$ nodes, and \(E\) is the set of edges.  
In the node classification task, each node $v \in V$ has an associated label $y$ and a $M$-length feature vector $x \in X$. Each edge \(e \in E\) refers to the nodes connection pair, denoted as \(e = (v_i, v_j)\), where \(i,j \in \{1,2,..., N\}\). 
We define $X_{N \times M}$ as the feature matrix of $V$, $Y_{N \times 1}$ as the labels set, and $A_{N \times N}$ as the adjacency matrix where \(A_{ij} = 1\) if \((v_i, v_j) \in E\) or \(i = j\), and 0 otherwise.
The degree of $v$ is the number of edges connected to it.

As shown in Figure \ref{fig:workflow_simCLR_}, graph contrastive learning methods generally follow a similar workflow \cite{chen2020simple,zhu2020deep,wang2022uncovering}.
First, two augmented views, $\tilde{G}'$ and $\tilde{G}''$, are generated from the given graph $G$ using transformation techniques like perturbation, removal, noise addition, etc.
We follow the transformation practice in ~\cite{zhu2020deep,srivastava2014dropout} to perform random masking on both edges and node features to form $\tilde{G}'$ and $\tilde{G}''$.

\begin{figure*}[t]
    \centering
    \includegraphics[width=1\linewidth]{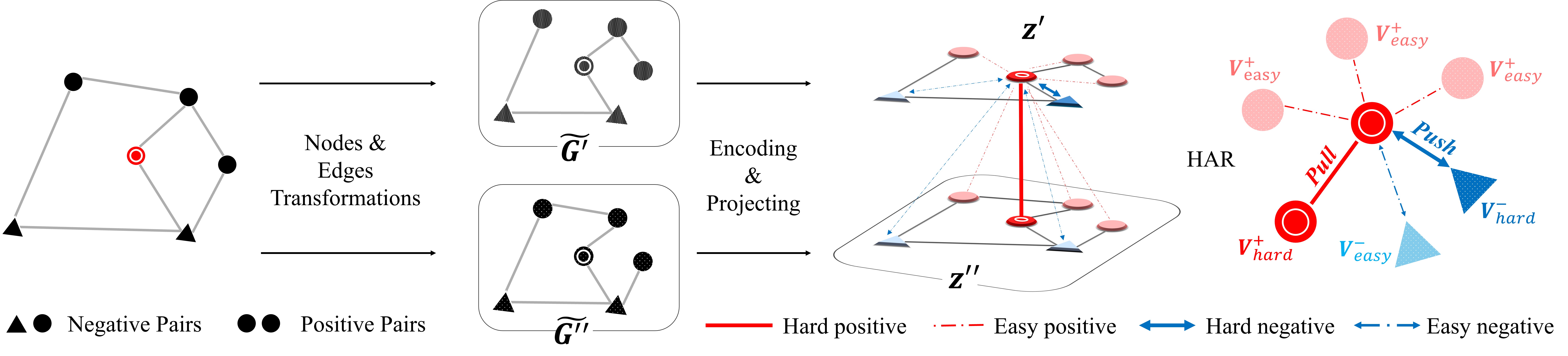}
    \caption{The Illustration of Graph Contrastive Learning with Our Proposed HAR Loss}
    \label{fig:flow_detailed}
\end{figure*}

\begin{itemize}
\item \textbf{Masking Edges (ME)}
Some edges are randomly removed from the original graph with an edge dropping rate $P_{e}$. We generate a random masking matrix $\tilde{R}_{N \times N} \in \{0,1\}$, where each entry is sampled from a Bernoulli distribution \(\tilde{R}_{ij} \sim \mathcal{B}(1 - p_{e})\) where $\tilde{R}_{ij} = 1$ if \(A_{ij} = 1\) in the original graph, and \(\tilde{R}_{ij} = 0\) otherwise. The new adjacency matrix \(\tilde{A}\) is obtained by performing the Hadamard product of \(A\) and \(\tilde{R}\), denoted as $ \tilde{A} = A \circ \tilde{R}$.

\item \textbf{Masking Features (MF)}
Similarly, features are randomly masked with a feature dropping rate $p_{f}$.
Each dimension of the features masking vector $\tilde{\textbf{r}} \in \{0,1\}^{M}$ is independently sampled following a Bernoulli distribution with probability $(1 - p_f)$. Each  node features $\tilde{x_i}$ are calculated by Hadamard product between $\tilde{\mathbf{r}}$ and each node feature vector $x_i$ where $i \in \{1,2, \ldots N\}$. Then the masked feature matrix is obtained by concatenating the masked feature vectors, denoted as $\tilde{X} = [x_1 \circ \tilde{\mathbf{r}}; x_2 \circ \tilde{\mathbf{r}}; \cdots; x_N \circ \tilde{\mathbf{r}}]^\top$.
\end{itemize}

Then, an encoder $H(\cdot)$ is used to produce data embeddings $h'$ and $h''$ for the augmented views. A projection head $Z(\cdot)$ further maps embeddings $h'$ and $h''$ into the space $z'$ and $z''$. Here we set GNNs as the encoder $H(\cdot)$ and MLP as the projector $Z(\cdot)$. Last, the contrastive loss is calculated from $z'$ and $z''$. In each training epoch, parameters are updated to minimize contrastive loss.

\begin{table*}[t]
\small
\centering
\caption{Hard/Easy and Positive/Negative Node Definitions for Each Given Node $v_i$. 
}
\renewcommand{\arraystretch}{1.2} 
\resizebox{0.95\textwidth}{!}{
\begin{tabular}{ll}
\hline
\textbf{Categories} & \textbf{Descriptions} \\
\hline
Hard Negative Nodes $\textbf{v}^-_{hard}$& Nodes with Different Labels $(y_i \neq y_j)$ but High $f_{\tau}(z_i,z_j)$ \\
\hline
Easy Negative Nodes $\textbf{v}^-_{easy}$ & Nodes with  Different Labels $(y_i \neq y_j)$ and Low $f_{\tau}(z_i,z_j)$\\
\hline
Hard Positive Nodes $\textbf{v}^+_{hard}$ &  Nodes with Same Label $(y_i = y_j)$ and Same Origin \\
\hline
Easy Positive Nodes $\textbf{v}^+_{easy}$& Nodes with Same Label $(y_i = y_j)$ but Different Origins\\
\hline
\end{tabular}
}
\label{tab:def_samples}
\end{table*}

\subsection{The Proposed Hardness Adaptive Reweighted Contrastive Loss}
\label{sec:method-gcl}

Figure \ref{fig:flow_detailed} shows the workflow of our proposed \underline{H}ardness \underline{A}daptive \underline{R}eweighted (HAR) contrastive loss. 
Our HAR adaptively assigns weights to each node’s positives and negatives based on nodes learning hardness. This novel weighting strategy can effectively address degree bias issue.
We first present our definitions of positives, negatives, and learning hardness, then describe HAR's adaptive weighting mechanism.

For a given node $v_i$, its relationship with other nodes $v_j$ $(j \in \{1,2,...,N\})$ can be categoriesed into hard negatives $\textbf{v}^-_{hard}$, easy negatives $\textbf{v}^-_{easy}$, hard positives $\textbf{v}^+_{hard}$, and easy positives $\textbf{v}^+_{easy}$, as defines in Table~\ref{tab:def_samples}. 
Nodes are treated as positives $\textbf{v}^+$ if they have the same label, and as negatives $\textbf{v}^-$ otherwise. 
We define the \textit{learning hardness} of $\textbf{v}^+$ based on augmentation origin\footnote{For any augmented node $v'_a \in \tilde{G}'$, its augmentation origin is defined as the corresponding original node $v_a \in G$. Two augmented nodes (eg. $v'_a \in \tilde{G}', v''_b \in \tilde{G}''$) have the same origin if their augmentation origins are the same $(v_a = v_b)$, and have different origins otherwise $(v_a \neq v_b)$.}, while that of $\textbf{v}^-$  on similarity between their representations $f_{\tau}(z_i, z_j)$ as defined in Equation (\ref{eq:ft_sim}). 
Intuitively, hard-to-learn nodes include nodes with the same origin (hard positives) and those highly similar but in different classes (hard negatives).

\begin{equation}
f_{\tau}(z_i, z_j)=\exp(\text{sim}(z_i, z_j)/\tau),
\label{eq:ft_sim}
\end{equation}

\begin{equation}
\text{sim}(z_i, z_j)= \frac{z_i \cdot z_j}{\|z_i\| \|z_j\|}.
\end{equation}

Based on the definitions, we can form Equation~\eqref{eq:mask_pos_neg} to obtain the overall positive mask matrix $\text{Mask}^+$ and negative mask matrix $\text{Mask}^-$ by identifying the positive and negative samples for each node. 
$\mathbb{I}$ is an indicator function, where $\mathbb{I}_{(y_i = y_j)}$ takes the value 1 iff $y_i$ and $y_j$ are the same, and 0 otherwise. Similar notation applies to $\mathbb{I}_{(y_i \neq y_j)}$.

{\small
\begin{equation}
\label{eq:mask_pos_neg}
\begin{split}
\mathrm{Mask}^+ = [\mathbb{I}_{(y_i = y_j)}]_{N \times N},
\mathrm{Mask}^- = [\mathbb{I}_{(y_i \neq y_j)}]_{N \times N}
\end{split}
\end{equation}
}

Equation~\eqref{eq:sim_intra_inter} defines the similarity matrix \( S \) to incorporate association information between nodes from two sources: \textbf{intra-view} associations (within the embedding representation from the same view, such as within \( z' \) or \( z'' \)) and \textbf{inter-view} associations (across different views, like between \( z' \) and \( z'' \)).

\begin{equation}
\label{eq:sim_intra_inter}
\begin{split}
S=[ s_{i,j} ]_{N \times N}, s_{i,j}= \underbrace{f_{\tau}(z'_i,z'_{j})}_{\text{inter-view}} + \underbrace{f_{\tau}(z'_i,z''_{j})}_{\text{intra-view}}
\end{split}
\end{equation}

\begin{figure*}[t]
    \centering
    \includegraphics[width=1\linewidth]{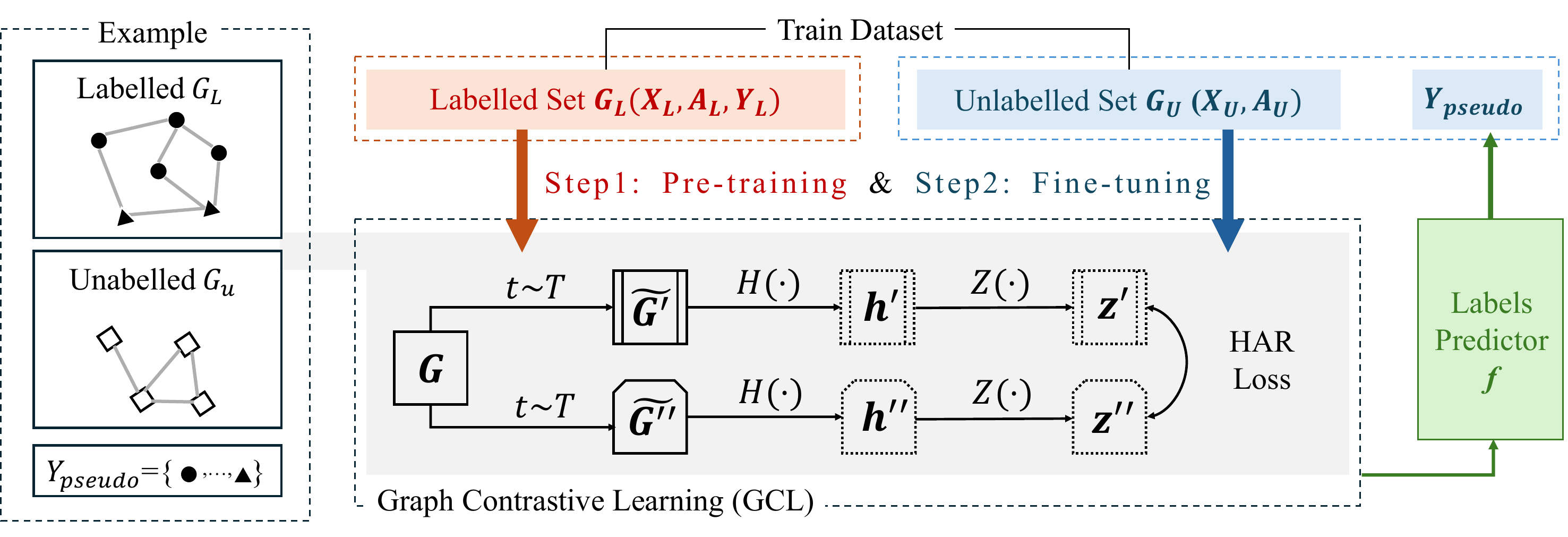}
    \caption{The Two-Step SHARP with HAR Loss on Node Classification Tasks}
    \label{fig:flow_all}
\end{figure*}

The corresponding similarity matrices for positive samples $S^+$ and negative samples $S^-$ can be formed by Equation~\eqref{eq:sim_pos_neg}, where $\circ$ denotes the Hadamard product. \( S^+_{i,j} \) and \( S^-_{i,j} \) means the values of \( S^+ \) and \( S^- \) at \( (i, j) \).
\begin{equation}
\label{eq:sim_pos_neg}
    S^+=S \circ \mathrm{Mask}^{+}, \space S^-=S \circ \text{Mask}^{-}.
\end{equation}

For positive pairs, we aim to encourage the model to focus more on learning the representation of the sample itself $\textbf{v}^+_{hard}$, while also considering information from other positive samples $\textbf{v}^+_{easy}$.
Therefore, we first scale $S^+$ to assign adaptive weights to all positive samples, then add extra weights to $\textbf{v}^+_{hard}$ via a $N\times N$ identity matrix $\mathbf{I}_{N}$. The reweighted positives are defined by Equation~\eqref{eq:pos_reweighted}.

\begin{equation}
W^{+}=\alpha \cdot \bar{S}^{+} + \mathbf{I}_{N},  
\end{equation}

{\small
\begin{equation}
\bar{S}^+ = [\bar{s}_{ij}]_{N \times N}, 
\bar{s}_{ij} = \frac{\max(S^+) - S_{ij}^+}{\max(S^+) - \min(S^+)}.
\end{equation}
}

\begin{equation}
\label{eq:pos_reweighted}
\mathrm{POS}_{i} = \sum\nolimits_{j=1}^{N} (  W^+_{ij} \cdot S^+_{ij})
\end{equation}

For negative samples, we refer to CL in debiasing \cite{chuang2020debiased} and reweighting \cite{kalantidis2020hard} strategies which assigns higher weights to $\textbf{v}^-_{hard}$ and lower weights to $\textbf{v}^-_{easy}$.

The reweighted negatives are defined by Equation~\eqref{eq:neg_reweighted}, where $\beta$ is the scaling term; $\tau^+$ is a class probability and is calculated as the reciprocal of the number of classes;
and $Q_{i}$ represents the ratio between positive and negative sample counts for node $v_i$, defined as $\frac{CNT_{i}({\mathrm{Mask}^-})}{CNT_{i}({\mathrm{Mask}^+})}$.
\begin{align*}
W^{-}=[ w^{-}_{ij} ]_{N \times N}, w^{-}_{i,j} = \frac{\beta \cdot  S^{-}_{ij} }{\frac{1}{N} \sum_{j=1}^{N}  S^{-}_{ij}}.
\end{align*}

{\small
\begin{equation}
\label{eq:neg_reweighted}
\mathrm{NEG}_{i} = 
\max (
\frac{ \sum_{j=1}^{N} (- Q_{i} \cdot \tau^{+} \cdot S^+ + W^- \cdot S^-) }{1 - \tau^{+}}, e^{-\frac{1}{\tau}} ) 
\end{equation}
}

After gaining the new weighted similarities for both positives and negatives, the loss for node $v_i$ is defined as Equation~\eqref{eq:ell_loss}.
\begin{equation}\label{eq:ell_loss}
\ell_{i}(z',z'')=-\log \frac{\mathrm{POS}_{i}  }
{
\mathrm{POS}_{i} +\mathrm{NEG}_{i} 
}
\end{equation}

Given the symmetry between the two contrasting perspectives, the overall HAR contrastive loss for the input graph data $G$ is designed as follows.
\begin{equation}\label{eq:HAR_loss}
HAR = \frac{1}{2N} \sum_{i=1}^N \left[ \ell_{i}(z',z'') + \ell_{i}(z'',z') \right]
\end{equation}

\textbf{Theoretical Analysis} We present the theoretical analysis of two hyper-parameters associated with the distributions of positive and negative samples based on previous analyses ~\cite{robinson2020contrastive}. 
Our findings show that even under worst-case hyperparameter settings, the misclassification risk of SHARP remains bounded. Specifically, SHARP degenerates to the case in \cite{robinson2020contrastive} when \(\alpha=0\), and for a fixed \(\beta\), we observe that the upper bound of the loss function is reached due to the monotonicity property of the logarithm function.
Furthermore, HAR loss does not exceed that of the aforementioned fully unsupervised objective function. The details are provided in Appendix \ref{sec:theoretical}.

\subsection{The Extension: Two-step SHARP Framework}
\label{sec:method-two-step}
To validate the effectiveness of our HAR loss in real-world cases where labels may not always be available, we design a two-step training framework that extends the application of our HAR loss to the semi-supervised scenario.
We call the proposed framework as SHARP (\underline{S}emi-supervised \underline{H}ardness \underline{A}daptive \underline{R}eweighted with \underline{P}seudo-labels) as shown in Figure~\ref{fig:flow_all}. 

Specifically, we split training data into two sets $G_{\text{train}} = G_{\text{train}}^{L} \cup G_{\text{train}}^{U}$ where labelled train set is denoted as $G_{\text{train}}^{L} = \{X_L, A_L, Y_L\}$ and unlabelled data is denoted as $G_{\text{train}}^{U} = \{X_U, A_U, \emptyset\}$.
The overall training process consists of two steps: First, the GCL is applied to the labelled data $G_{\text{train}}^{L}$ to generate a pre-trained model. Second, the pre-trained model is used to generate the pseudo-labels for $G_{\text{train}}^{U} $.
The new $G_{\text{train}}^{U} = \{ X_{U}, A_{U},Y_{pseudo} \}$ is then fed into the GCL for further fine-tuning the pre-trained model. Both steps are in the same GCL workflow with HAR loss, using their respective datasets.

\section{Experiments and Results}
This section first describes the experimental setup, followed by a global performance comparison and a detailed analysis across degree groups. We last present our parameter sensitivity analysis and dimension reduction visualization to confirm the effectiveness of our method, supporting the validity of the results.

\subsection{Experiment Settings} 

\textbf{Datasets} The experiments are run on four datasets with node classification tasks: Cora \cite{mccallum2000automating}, CiteSeer \cite{giles1998citeseer}, PubMed \cite{sen2008collective}, and Wiki-CS \cite{mernyei2020wiki}.
The first three datasets are paper citation networks on different domains, Wiki-CS contains Wikipedia articles in computer science topics, and all aiming to predict specific topic of each node.
\begin{table}[http]
  \centering
  \caption{Dataset Statistics}
  \label{tab:table1}
  \resizebox{0.5\textwidth}{!}{
  \begin{tabular}{cccccc}
    \hline
    Dataset & Type & Nodes & Edges & Classes & Features \\
    \hline
        Cora & Citation & 2,708 & 5,429 & 7 & 1,433 \\
    Citeseer & Citation & 3,327 & 4,732 & 6 & 3,703 \\
    Pubmed & Citation & 19,717 & 44,338 & 3 & 500 \\
    Wiki-CS & Wikipedia & 11,701 & 216,123 & 10 & 300 \\
    \hline
  \end{tabular}}
\end{table}

\textbf{Graph Contrastive Learning Settings} Our graph learning processing is based on PyG package \cite{Fey/Lenssen/2019}.
We run experiments on both two-layer GCN and GAT as the encoder, the projector is set to be a two-layer Multilayer Perceptron (MLP). 
The contrastive loss is implemented as the proposed HAR. 
The parameters setting like edge dropping rate $p_e$, feature dropping rate $p_f$, temperature $\tau$, etc., are listed in the Appendix \ref{apx:experiment}. 

\begin{table*}[t]
\caption{Prediction Performance Comparison across Different Models on F1-Score}
\resizebox{1\textwidth}{!}{
\begin{tabular}{ccccc|ccccc}
\toprule
\multicolumn{5}{c}{GCN as Graph Layers}         & \multicolumn{5}{c}{GAT as Graph Layers}           \\
\midrule
\rowcolor[HTML]{EFEFEF}  $r=0$  & Cora   & CiteSeer & PubMed & Wiki-CS  & $r=0$   & Cora   & CiteSeer & PubMed & Wiki-CS  \\
GCN       & 0.8340 & 0.7140   & 0.8630 & 0.8116   & GAT       & 0.8345 & 0.7230   & 0.8690 & 0.8214   \\
GRACE     & 0.7992 & 0.6914   & 0.8392 & 0.7487 & GRACE     & 0.7928 & 0.6790    & 0.8388 & 0.7568 \\
SCL       & 0.8176 & 0.7202   & 0.8926 & 0.8027   & SCL       & 0.8415 & 0.7218   & 0.8814 & 0.8070   \\
Debias    & 0.8152 & 0.7226   & 0.8906 & 0.8062   & Debias    & 0.8465 & 0.7196   & 0.8837 & 0.8095   \\
\midrule
SHARP(Ours) & \textbf{0.8700} & \textbf{0.7252}   & \textbf{0.8942} & \textbf{0.8193}   & SHARP(Ours) & \textbf{0.8770} & \textbf{0.726}    & \textbf{0.8931} & \textbf{0.8235}  \\
\midrule
\rowcolor[HTML]{EFEFEF} $r=0.3$  & Cora   & CiteSeer & PubMed          & Wiki-CS & $r=0.3$  & Cora   & CiteSeer & PubMed          & Wiki-CS \\ 
GCN    & 0.8170 & 0.7120   & 0.8740          & 0.8078  & GAT    & 0.8260 & 0.7090   & 0.8640          & 0.7975  \\
GRACE  & 0.7839 & 0.6938   & 0.8354          & 0.7499  & GRACE  & 0.7948 & 0.6976   & 0.8410          & 0.7513  \\
SCL    & 0.7804 & 0.7117   & 0.8860          & 0.8002  & SCL    & 0.8318 & 0.7188   & \textbf{0.8830} & 0.8017  \\
Debias & 0.7846 & 0.7100   & \textbf{0.8872} & 0.8011  & Debias & 0.8254 & 0.7224   & 0.8814          & 0.8007  \\ 
\midrule
SHARP(Ours) & \textbf{0.8586} & \textbf{0.7212} & 0.8856 & \textbf{0.8150} & SHARP(Ours) & \textbf{0.8667} & \textbf{0.7268} & 0.8809 & \textbf{0.8191} \\ 
\bottomrule
\end{tabular}
}
\label{tab:global}
\end{table*}

\textbf{Two-step Learning Settings}
We follow \cite{sen2008collective,kipf2016semi} to split the data into train $G_{\text{train}}$, validation $G_{\text{val}}$ and test $G_{\text{test}}$ datasets. 
Then, $G_{\text{train}}$ is splitted into two parts: a fraction $r \in[0,1)$ forms the unlabelled sub-training set $G_{\text{train}}^\text{U}$, while the remainder forms the labelled sub-training set $G_{\text{train}}^\text{L}$.
First, the model is pre-trained with $G_{\text{train}}^\text{L}$ for $epoch$ rounds. In the following model fine-tuning on the $G_{\text{train}}^\text{U}$, we validate the prediction performance on the validation set $G_{\text{val}}$ for each round and record the best history result. When the best result is not updated for \( K \) rounds, we stop the training and save the encoder. The model evaluation is based on embeddings of $G_{test}$ generated from this saved encoder.

\textbf{Baselines}
One GNN-based and three GCL-based methods are set as baselines. All GCL-based methods follow the same GCL workflow and with different contrastive losses (GRACE \cite{zhu2020deep}, SCL \cite{khosla2020supervised}, Debias \cite{chuang2020debiased}). The final embeddings are then fed into a logistic regression classifier for performance evaluation.
To ensure results comparability, all baselines follow the same two-step semi-supervised learning: pre-training their respective initial performance on the $G_{\text{train}}^\text{L}$, and fine-tuning on the $G_{\text{train}}^\text{U}$.

\textbf{Evaluation} 
All evaluations are under the test set $G_{\text{test}}$ using F1-score as the metric. Each model is run 10 times, and we report the mean scores.
F1-score is set as our evaluation metric as it is derived from Precision and Recall—which in turn are calculated based on the values of True Positives (TP), False Positives (FP), False Negatives (FN), and True Negatives (TN) in the confusion matrix—providing a more balanced measure of performance.
We perform both global and degree-level evaluations: 
(1) The global evaluation computes performance across the entire test set. 
(2) For degree-level evaluation, we group test data by node degrees and compute F1-score within each degree group.
The degree bias is assessed by comparing performance disparities among degree groups.

\subsection{Result Discussions}
The results show SHARP’s performance at the global and degree levels. The robustness of SHARP is examined through experiments with different hyperparameter combinations and latent embedding visualizations.

\subsubsection{Global Level Performance}

Table \ref{tab:global} illustrates the global F1-score performance of five models across four datasets.
Compared to GRACE, which only treats one single positive pair, SCL, Debias, and SHARP all utilize labels to customize positives and negatives for each node, and their overall performance is better than GRACE.
Our SHARP outperforms the baseline in 14 out of 16 tests and shows most significant improvement on the Cora dataset, with a 3.6\% increase in performance under the GCN layer ($r=0$) and a 4.2\% increase under the GAT layer ($r=0$) setting. 
SHARP performs slightly lower than SCL and Debias on the PubMed dataset ($r=0.3$). This can be due to our current simple semi-supervised scheme for all baselines, which applies a fixed-step early stopping strategy during the second-step fine-tuning process, potentially introducing fluctuations to the results. We focus on how GCL settings benefits graph structural fairness in this paper, and leave the discussion of different fine-tuning strategies in semi-supervised learning for future work.

\begin{figure}[b]
    \centering
    \includegraphics[width=1\linewidth]{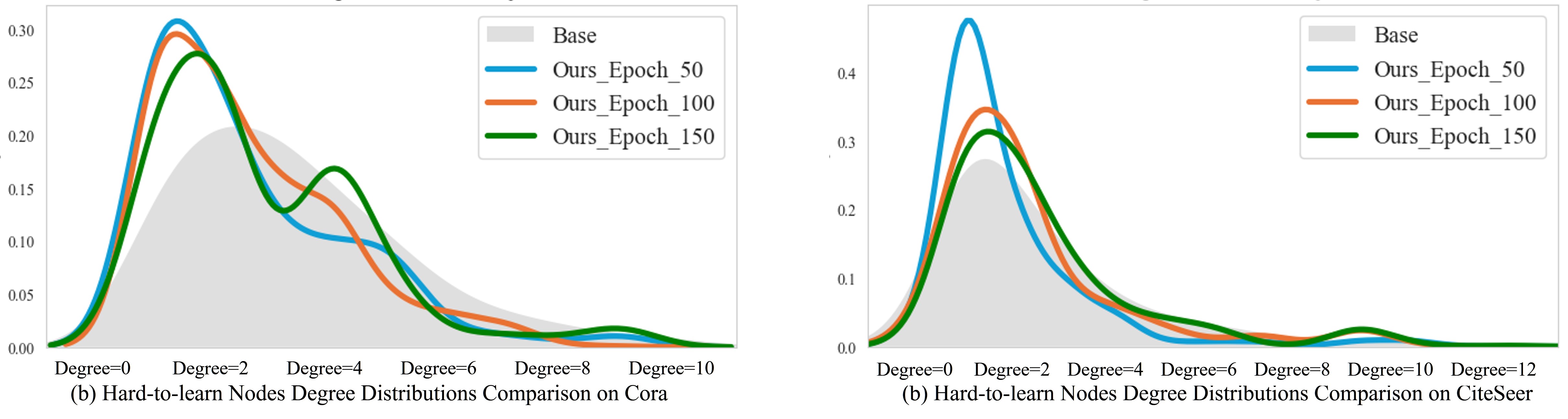}
    \caption{The Node Degree Distributions of Top-5 Hard Negatives (GAT layers, $r=0$ )} 
    \label{fig:hard-to-lean}
\end{figure}

\subsubsection{Degree Level Performance}
Besides overall performance, we also analyze degree bias by comparing performance across degree groups. Small sample sizes in some very high-degree groups can reduce performance confidence. Thereby we focus on nodes with degrees 0 to 7 to ensure sufficient samples within degree groups.
Figure \ref{fig:ablation}(a) illustrates degree-level performance on the Cora dataset ($r=0.3$). Among all methods, SHARP shows the most significant improvement for low-degree nodes, achieving over 10\% increase for the node group (degree=1) compared to the GAT baseline.
Comparing growth rates across subgroups, SHARP enhancement for low-degree nodes is higher than high-degree group.
Similar trends are also in PubMed datasets (Table \ref{tab:degree_pubmed_gat}). Despite limited global improvement discussed above, subgroup analysis shows SHARP focuses on boosting scores for low-degree nodes, narrowing the performance gap with high-degree nodes.
Therefore, we conclude our method shows a targeted improvement, particularly benefiting low-degree nodes across datasets, addressing the uneven performance on different node degree groups.

\begin{table*}[t]
\centering
\caption{Degree-level Prediction Performance Comparison on PubMed Dataset (GAT as the graph layers, $r=0$, $\Delta$ indicates changes compared with the GAT baseline)}
\resizebox{1\textwidth}{!}{
\begin{tabular}{cc|c|cc|cc|cc|cc}
\toprule
Degree & Size & GAT & GRACE & $\Delta$ & SCL & $\Delta$ & Debias & $\Delta$  & SHARP(Ours) & $\Delta$ \\
\midrule
1 & 483 & 0.8364 & 0.8023 & -3.41\% & 0.8553 & 1.89\%          & 0.8553 & 1.89\%          & 0.8706 & \textbf{3.42\%} \\
2 & 153 & 0.8824 & 0.8471 & -3.53\% & 0.9000 & 1.76\%          & 0.9007 & 1.83\%          & 0.9190 & \textbf{3.66\%} \\
3 & 79  & 0.9367 & 0.9051 & -3.16\% & 0.9519 & 1.52\%          & 0.9481 & 1.14\%          & 0.9519 & \textbf{1.52\%} \\
4 & 40  & 0.9250 & 0.8600 & -6.50\% & 0.9500 & 2.50\%          & 0.9500 & 2.50\%          & 0.9750 & \textbf{5.00\%} \\
5 & 35  & 0.8857 & 0.8600 & -2.57\% & 0.9143 & \textbf{2.86\%} & 0.9229 & \textbf{3.72\%} & 0.8971 & 1.14\%          \\
6 & 28  & 0.8571 & 0.8750 & 1.79\%  & 0.8929 & \textbf{3.58\%} & 0.8929 & \textbf{3.58\%} & 0.8929 & \textbf{3.58\%} \\
7 & 14  & 0.9286 & 0.9357 & 0.71\%  & 1.0000 & \textbf{7.14\%} & 0.9714 & 4.28\%          & 0.9286 & 0.00\%
\\ \bottomrule
\end{tabular}
}
\label{tab:degree_pubmed_gat}
\end{table*}

\begin{figure*}[t]
    \centering
    \includegraphics[width=1\linewidth]{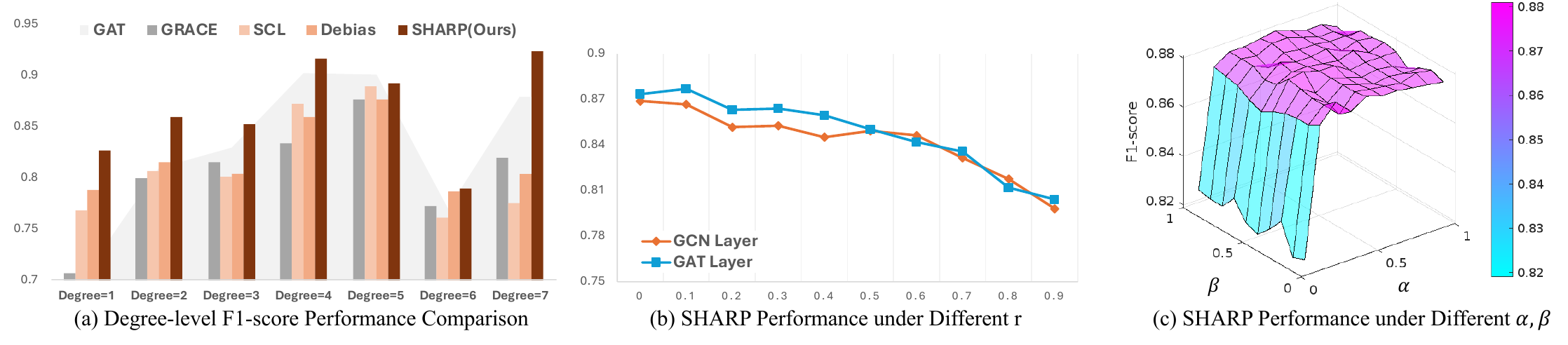}
    \caption{(a) Degree Bias Comparison; (b)(c) $r, \alpha, \beta$ Sensitivity Analysis on Cora Dataset}
\label{fig:ablation}
\end{figure*}

To further investigate why our method can reduce degree bias, we further analyze the degree distributions of hard-to-learn nodes. For each node, the top hard-to-learn nodes (those assigned higher weights) are identified. Then, the degrees of these selected nodes are recorded and their degree distributions are visualized using Kernel Density Estimation\footnote{The KDE plot smooths data with a Gaussian kernel for a continuous density estimate.}.

Figure \ref{fig:hard-to-lean} displays degree distribution comparisons of Top-5 hard negatives across different training epochs. The gray area represents the original distribution without adaptive weights, serving as our ablation study. 
The proportion of low-degree nodes in the hard-to-learn nodes distribution is higher than the base distribution. Assigning higher weights to hard-to-lean nodes puts more focus on low-degree nodes during training and helps mitigate degree bias. Traditional equal weighting fails to explore the full potential of low-degree nodes.

\begin{figure*}[t]
    \centering
    \includegraphics[width=1\linewidth]{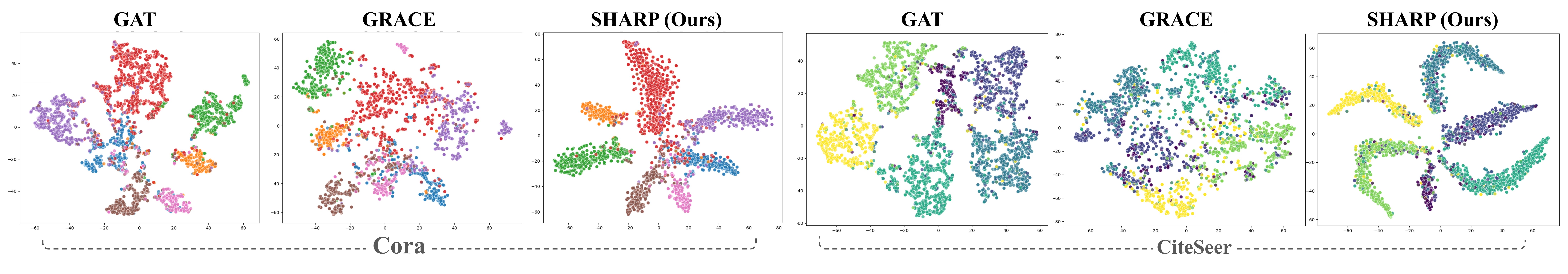}
        \caption{The t-SNE Visualizations of Embeddings on Cora and CiteSeer Dataset}
    \label{fig:vis_tsne}
\end{figure*}

\subsubsection{Parameters Sensitivity Analysis}
To assess parameter sensitivity, additional experiments are performed with different proportions of unlabelled data \( r \in [0,1) \), and the hyperparameters \( \alpha \in (0,1] \) and \( \beta \in (0,1] \). Figure \ref{fig:ablation}(b)(c) shows the comparison results on the Cora dataset.
The results are generally robust to the scaling parameters \(\alpha\) and \(\beta\), with performance varying by less than 2\% across different parameter combinations.
However, when \(\beta\) is particularly small (close to 0), a noticeable decline occurs. This implies that ignoring negative samples can cause the loss function to rely solely on positive samples, and can hinder effective contrastive learning.
The performance increases as $r$ decreases, which aligns with our intuition that more data labelled can provide more information that benefits learning. A lower $r$ indicates a higher proportion of labelled data in the first pre-training step, enabling the model to learn more feature patterns. Still, our model shows competitive performance despite of the performance drop with higher $r$ values.
Despite a performance drop with higher $r$ values, our model still shows competitive performance. Notably, when using SHARP (with GCN layers) with only half labelled data $(r=0.5)$, we still achieve an F1-score of 0.8504. This surpasses the performance of 0.834 which uses GCNs alone with the entire labelled dataset for training.

\subsubsection{Visualisations}

To further compare the structure of the models' learned representations, we apply t-SNE \cite{van2008visualizing} to map the high-dimensional embedding features to a low-dimensional space for visualization. Figure \ref{fig:vis_tsne} shows the visualisations (GAT Layer, $r=0.3$), with the full comparison in Appendix \ref{apx:experiment}. 
Each point represents the t-SNE processed latent space of the node embedding, with colors indicating different label categories.
SHARP displays the clearest decision boundaries of the seven categories in Cora and six categories in CiteSeer. For inter-class performance, SHARP pushes nodes from different classes further apart than other methods. It aligns with the principle of contrastive learning by distancing negative samples and validates the effectiveness of our designed HAR loss. For intra-class performance, SHARP also pulls same-class nodes more tightly together than other methods. It reduces degree bias by strengthening the connections of low-degree nodes to their class. We conclude that our SHARP can capture the underlying structure of the data more effectively. The superior interclass and intraclass performance further validates SHARP's effectiveness in degree bias mitigation.

\section{Conclusions}
The paper proposes a novel contrastive loss called HAR to mitigate degree bias in node classification.
HAR offers two main advantages: (1) it assigns more positives to each node, providing low-degree nodes with additional information; and (2) it assigns weights to positives and negatives adaptive to their learning hardness, allowing more focus on low-degree nodes and more information about them to be learned.

Both theoretical analysis and experiments validate HAR’s effectiveness. In the experiments, we design a two-step pipeline named SHARP to integrate and examine HAR’s applicability in more diverse scenarios.
Experiments on four datasets and comparisons with four baselines show that by focusing more on hard-to-learn nodes, the proposed SHARP can maintain good global performance while enhancing local performance for low-degree nodes, thereby reducing degree bias. 

One limitation of the current work is that when the proportion of labelled data is particularly low (e.g., less than 30\%), it can restrict the improvement of SHARP. One possible reason for this is that insufficient data can lead to a lack of feature information, which impacts graph representation learning. Future work could further explore using pre-trained models to add richer semantic information to low-degree nodes \cite{zhao2022ap,yan2023comprehensive}, as well as extend this approach to more diverse scenarios such as heterogeneous graphs.

\bibliography{anthology,a-ref}
\bibliographystyle{acl_natbib}

\appendix
\onecolumn


\section{Detailed Related Work of Graph Representation Learning }
This section provides a detailed discussion of related work.
Graph representation learning embeds graph-structured data into a feature space. The learned graph representations are generally integrated with task-specific neural networks to perform downstream tasks at various levels: graph-level, link-level, or node-level. This paper focuses on node-level tasks, which aim to predict the label of each node in a graph. GNNs and GCL are two common ways of generating embeddings for node classification tasks.

\noindent
\textbf{Graph Neural Networks (GNNs)} are proposed to extend neural networks to process graph-structured data which has nodes features and connections between nodes. The core idea of GNNs is message passing between nodes, updating feature values by considering connected nodes layer by layer. Graph layers (GLs) define how nodes send, receive, and update messages. Convolution-based (e.g., GCN \cite{kipf2016semi}), attention-based (e.g., GAT \cite{velivckovic2017graph}), and aggregation-based (e.g., GraphSAGE \cite{hamilton2017inductive}) are three common types of graph layers~\cite{bronstein2021geometric}.

GCN aggregates using constant-weighted averages of neighbor messages. 
The layer-wise propagation rule of the GCN is as follows, where \( \tilde{A} \) refers to the adjacency matrix $A$ plus self-loops, \( \tilde{D} \) is the degree matrix of \( \tilde{A} \), \( W^{(l)} \) is the trainable weight matrix, and \( H^{(l)} \) is the activation matrix at layer \( l \). \( H^{(0)} \) equals the feature matrix \( X \) at layer \( l=0 \).

\begin{equation}
H^{(l+1)} = \sigma\left(  \hat{A} H^{(l)} W^{(l)} \right), \text{where} \space \hat{A} = \tilde{D}^{-\frac{1}{2}}  \tilde{A} \tilde{D}^{-\frac{1}{2}} 
\end{equation}

Based on the idea of GCN, Graph Attention Network (GAT) adds attention mechanisms to learn the importance of neighboring nodes when aggregating feature information. GraphSAGE samples a subset of neighbors to reduce computational complexity for nodes with massive degrees.

\noindent
\textbf{Graph Contrastive Learning (GCL)} combines GNNs with Contrastive Learning (CL) which enhances the robustness of learned representations. CL is a learning scheme that focuses on pulling similar/positive samples closer while pushing dissimilar/negative samples away using a contrastive loss \cite{chen2020simple}. 
There are different contrastive loss definitions to ensure positive samples are brought together while negative samples are distanced effectively in GCL.
The classic SimCLR \cite{chen2020simple} extends the usage of InfoNCE \cite{oord2018representation} in contrastive learning, aiming to learn representations by maximizing agreement between differently augmented views of the same sample in the latent space. Supervised contrastive learning (SCL) \cite{khosla2020supervised} uses label information to more accurately identify positive samples as those from the same class, while negative samples are those from different classes. \cite{assran2020supervision,li2021comatch} share a similar idea in the semi-supervised setting. Instead of treating all samples as equally important, \cite{robinson2020contrastive} proposes to assign adaptive weights for negative samples when calculating contrastive loss. Specifically, they focus on hard negative samples, which have high similarity but actually belong to different classes. They prove that these samples are harder to learn from and are worth assigning higher weights.

Specifically, for a given graph $G_i \in G$, SimCLR treats differently augmented views of the same sample as the positive pair $\{\tilde{z_i}',\tilde{z_i}''\}$, the remaining $2 \times (N - 1)$ augmented samples are treated as negatives, denoted as \( G_{\text{neg}} = \{G \setminus G_i\} \).
Equation~\eqref{eq:loss_simCLR} defines the SimCLR loss for the positive pair $\{i,j\}$, \( \tau \) is a temperature parameter used to scale the \( \text{sim}(\cdot, \cdot) \), and \( \mathbb{I}_{[k \neq i]} \) is an indicator function used to exclude comparisons with itself.
\begin{equation}
\ell_{i,j}=-\log \frac{f_{\tau}(z_i,z_j)}{\sum_{k=1}^{2N}\mathbb{I}_{[k\neq i]}f_{\tau}(z_i,z_k)}
\label{eq:loss_simCLR}
\end{equation}

In node classification task, GRACE \cite{zhu2020deep} generates contrastive views by randomly masking both nodes and edges, and considers the inter-view and intra-view negative pairs when designing the contrastive loss. For each node \( v_i \in V \), let the representations of the augmented views be denoted as \( z'_i \) and \( z''_i \). The GRACE loss is then defined as Equation~\eqref{eq:grace_loss}.

\begin{equation}
\ell(z'_i,z''_i)=-\log \frac{f_{\tau}(z'_i,z''_i)}{f_{\tau}(z'_i,z''_i)+ \sum\limits_{k=1}^N \mathbb{I}_{[k\neq i]}f_{\tau}(z'_i,z'_k)+  \sum\limits_{k=1}^N \mathbb{I}_{[k\neq i]}f_{\tau}(z'_i,z''_k)}
\label{eq:grace_loss}
\end{equation}

\section{Details of Experiments}
\label{apx:experiment}

\textbf{Settings}
We fixed the weight decay to $1e-5$. The feature dropping rate $p_f$, and edge drop rate $p_e$, are in the range [0, 0.4], the reweighting factors $\alpha$ and $\beta$ are in the range (0, 1]. The early stop $K$ rounds is set to 20. More differentiated parameter settings across different datasets are shown in Table \ref{tab: settings}, where `Neurons' represents the number of neurons in the hidden and project hidden layers.

\begin{table}[http]
  \centering
  \caption{Parameters Settings on Different Datasets}
  \label{tab: settings}
  \resizebox{0.5\textwidth}{!}{
  \begin{tabular}{cccccc}
    \hline
    Dataset  & $\tau$ & Neurons & Epoch & Learning Rate & Activation \\ \hline
Cora     & 0.4    & 128     & 300   & 5e-4  & relu        \\
Citeseer & 0.9    & 256     & 300   & 1e-3   & prelu       \\
Pubmed   & 0.7    & 256     & 350   & 1e-3   & relu       \\
Wiki-CS  & 0.4    & 256     & 350   & 1e-2  & relu        \\     \hline
  \end{tabular}
  }
\end{table}

The evaluation metrics used in the paper are defined as below.
\begin{align}
\label{eq:f-score}
\text{Precision} &= \frac{TP}{TP + FP}, 
\text{Recall} = \frac{TP}{TP + FN}, 
\text{F1-score} = \frac{2 \cdot \text{Precision} \cdot \text{Recall}}{\text{Precision} + \text{Recall}}
\end{align}

\noindent
\textbf{Degree Level Results}
The tables below provide supplementary results for degree-level performance analysis on Cora and PubMed datasets.
For the Cora dataset, SHARP displayed the most significant global performance improvement. At the degree-level analysis, it also improved predictions for low-degree nodes.
For the PubMed dataset, SHARP showed least global performance improvements. However, at the degree-level analysis, SHARP still improved the prediction of low-degree node.

\begin{table*}[http]
\centering
\caption{Comparison of Degree-level Prediction F1-score Across Models on Cora Dataset (GCN as the foundation graph layer, $r=0$, $\Delta$ indicates difference with GCN)}
\resizebox{\textwidth}{!}{
\begin{tabular}{cc|c|cc|cc|cc|cc}
\toprule
Degree & Size & GCN & GRACE & $\Delta$ & SCL & $\Delta$ & Debias & $\Delta$  & SHARP(Ours) & $\Delta$ \\
\midrule
1      & 161  & 0.7143   & 0.7323                    & 1.80\%                                          & 0.7565                  & 4.22\%                                        & 0.7758                     & 6.15\%                                           & 0.8081                        & \textbf{9.38\%}                               \\
2      & 226  & 0.8319   & 0.8208                    & -1.11\%                                         & 0.8000                     & -3.19\%                                       & 0.8128                     & -1.91\%                                          & 0.8504                        & \textbf{1.85\%}                               \\
3      & 223  & 0.8341   & 0.7987                    & -3.54\%                                         & 0.783                   & -5.11\%                                       & 0.7892                     & -4.49\%                                          & 0.8744                        & \textbf{4.03\%}                               \\
4      & 134  & 0.9030    & 0.8373                    & -6.57\%                                         & 0.8299                  & -7.31\%                                       & 0.8507                     & -5.23\%                                          & 0.9194                        & \textbf{1.64\%}                               \\
5      & 101  & 0.8713   & 0.8861                    & 1.48\%                                          & 0.8337                  & -3.76\%                                       & 0.8673                     & -0.40\%                                          & 0.9050                        & \textbf{3.37\%}                               \\
6      & 52   & 0.8077   & 0.7385                    & -6.92\%                                         & 0.7212                  & -8.65\%                                       & 0.7615                     & -4.62\%                                          & 0.7865                        & -2.12\%                              \\
7      & 25   & 0.8000      & 0.8280                     & 2.80\%                                          & 0.7960                   & -0.40\%                                       & 0.8240                      & 2.40\%                                           & 0.8720                         & \textbf{7.20\%}                               \\ \bottomrule
\end{tabular}
}
\end{table*}

\begin{table*}[http]
\centering
\caption{Comparison of Degree-level Prediction F1-score Across Models on Cora Dataset (GCN as the foundation graph layer, $r=0.3$, $\Delta$ indicates difference with GCN)}
\resizebox{\textwidth}{!}{
\begin{tabular}{cc|c|cc|cc|cc|cc}
\toprule
Degree & Size & GCN & GRACE & $\Delta$ & SCL & $\Delta$ & Debias & $\Delta$  & SHARP(Ours) & $\Delta$ \\
\midrule
1      & 161  & 0.7081   & 0.7012                    & -0.69\%                                         & 0.7665                  & 5.84\%                                        & 0.7466                     & 3.85\%                                           & 0.8174                        & \textbf{10.93\%}                              \\
2      & 226  & 0.8009   & 0.7867                    & -1.42\%                                         & 0.8049                  & 0.40\%                                        & 0.7872                     & -1.37\%                                          & 0.8358                        & \textbf{3.49\%}                               \\
3      & 223  & 0.8206   & 0.7623                    & -5.83\%                                         & 0.783                   & -3.76\%                                       & 0.7677                     & -5.29\%                                          & 0.8538                        & \textbf{3.32\%}                               \\
4      & 134  & 0.9104   & 0.8224                    & -8.80\%                                         & 0.8582                  & -5.22\%                                       & 0.8313                     & -7.91\%                                          & 0.9112                        & \textbf{0.08\%}                               \\
5      & 101  & 0.8416   & 0.8663                    & 2.47\%                                          & 0.8871                  & 4.55\%                                        & 0.8653                     & 2.37\%                                           & 0.9079                        & \textbf{6.63\%}                               \\
6      & 52   & 0.7308   & 0.7365                    & 0.57\%                                          & 0.7423                  & 1.15\%                                        & 0.7173                     & -1.35\%                                          & 0.7654                        & \textbf{3.46\%}                               \\
7      & 25   & 0.8400     & 0.7760                     & -6.40\%                                         & 0.8160                   & -2.40\%                                       & 0.8000                        & -4.00\%                                          & 0.9280                         & \textbf{8.80\%}                               \\ \bottomrule
\end{tabular}
}
\end{table*}

\begin{table*}[http]
\centering
\caption{Comparison of Degree-level Prediction Performance on Cora Dataset (GAT as the graph layer, $r=0$, $\Delta$ indicates difference with the baseline GAT)}
\resizebox{\textwidth}{!}{
\begin{tabular}{cc|c|cc|cc|cc|cc}
\toprule
Degree & Size & GAT & GRACE & $\Delta$ & SCL & $\Delta$ & Debias & $\Delta$  & SHARP(Ours) & $\Delta$ \\
\midrule
1      & 161  & 0.7143   & 0.7137                    & -0.06\%                                         & 0.7702                  & 5.59\%                                        & 0.7944                     & 8.01\%                                           & 0.8106                        & \textbf{9.63\%}                               \\
2      & 226  & 0.8186   & 0.8181                    & -0.05\%                                         & 0.815                   & -0.36\%                                       & 0.8283                     & 0.97\%                                           & 0.869                         & \textbf{5.04\%}                               \\
3      & 223  & 0.8475   & 0.8027                    & -4.48\%                                         & 0.8045                  & -4.30\%                                       & 0.8323                     & -1.52\%                                          & 0.8857                        & \textbf{3.82\%}                               \\
4      & 134  & 0.8955   & 0.8455                    & -5.00\%                                         & 0.8478                  & -4.77\%                                       & 0.8672                     & -2.83\%                                          & 0.9269                        & \textbf{3.14\%}                               \\
5      & 101  & 0.8713   & 0.8782                    & 0.69\%                                          & 0.8495                  & -2.18\%                                       & 0.8693                     & -0.20\%                                          & 0.9059                        & \textbf{3.46\%}                               \\
6      & 52   & 0.8462   & 0.7558                    & -9.04\%                                         & 0.7577                  & -8.85\%                                       & 0.7731                     & -7.31\%                                          & 0.7846                        & -6.16\%                              \\
7      & 25   & 0.8800     & 0.8240                     & -5.60\%                                         & 0.7680                   & -11.20\%                                      & 0.7760                      & -10.40\%                                         & 0.8800                          & \textbf{0.00\%}                               \\ \bottomrule
\end{tabular}
}
\end{table*}

\begin{table*}[http]
\centering
\caption{Comparison of Degree-level Prediction F1-score Across Models on Cora Dataset (GAT as the foundation graph layer, $r=0.3$, $\Delta$ indicates changes compared with GAT)}
\resizebox{0.98\textwidth}{!}{
\begin{tabular}{cc|c|cc|cc|cc|cc}
\toprule
Degree & Size & GAT & GRACE & $\Delta$ & SCL & $\Delta$ & Debias & $\Delta$  & SHARP(Ours) & $\Delta$ \\
\midrule
1      & 161  & 0.7081   & 0.7075                    & -0.06\%                                         & 0.7689                  & 6.08\%                                        & 0.7882                     & 8.01\%                                           & 0.8273                        & \textbf{11.92\%}                              \\
2      & 226  & 0.8097   & 0.7996                    & -1.01\%                                         & 0.8066                  & -0.31\%                                       & 0.8155                     & 0.58\%                                           & 0.8597                        & \textbf{5.00\%}                               \\
3      & 223  & 0.8296   & 0.8148                    & -1.48\%                                         & 0.8018                  & -2.78\%                                       & 0.8036                     & -2.60\%                                          & 0.8525                        & \textbf{2.29\%}                               \\
4      & 134  & 0.9030    & 0.8336                    & -6.94\%                                         & 0.8724                  & -3.06\%                                       & 0.8604                     & -4.26\%                                          & 0.9164                        & \textbf{1.34\%}                               \\
5      & 101  & 0.9010    & 0.8762                    & -2.48\%                                         & 0.8891                  & -1.19\%                                       & 0.8772                     & -2.38\%                                          & 0.8921                        & -0.89\%                              \\
6      & 52   & 0.7692   & 0.7731                    & 0.39\%                                          & 0.7615                  & -0.77\%                                       & 0.7865                     & 1.73\%                                           & 0.7904                        & \textbf{2.12\%}                               \\
7      & 25   & 0.8800     & 0.8200                      & -6.00\%                                         & 0.7760                   & -10.40\%                                      & 0.8040                      & -7.60\%                                          & 0.9240                         & \textbf{4.40\%}                               \\ \bottomrule
\end{tabular}
}
\label{tab:degree_cora_gat}
\end{table*}

\begin{table*}[http]
\centering
\caption{Comparison of Degree-level Prediction F1-score Across Models on PubMed Dataset (GCN as the foundation graph layer, $r=0$, $\Delta$ indicates difference with GCN)}
\resizebox{\textwidth}{!}{
\begin{tabular}{cc|c|cc|cc|cc|cc}
\toprule
Degree & Size & GCN & GRACE & $\Delta$ & SCL & $\Delta$ & Debias & $\Delta$  & SHARP(Ours) & $\Delta$ \\
\midrule
1 & 483 & 0.8406 & 0.8077 & -3.29\% & 0.8685 & 2.79\%           & 0.8638 & 2.32\%           & 0.8712 & \textbf{3.06\%}  \\
2 & 153 & 0.8824 & 0.8497 & -3.27\% & 0.9307 & \textbf{4.83\%}  & 0.9275 & 4.51\%           & 0.9157 & \textbf{3.33\%}  \\
3 & 79  & 0.9494 & 0.9203 & -2.91\% & 0.9658 & 1.64\%           & 0.9557 & 0.63\%           & 0.9620 & \textbf{1.26\%}  \\
4 & 40  & 0.9000 & 0.8650 & -3.50\% & 0.9500 & 5.00\%           & 0.9500 & 5.00\%           & 0.9675 & \textbf{6.75\%}  \\
5 & 35  & 0.9143 & 0.8686 & -4.57\% & 0.9514 & \textbf{3.71\%}  & 0.9400 & \textbf{2.57\%}  & 0.9286 & 1.43\%           \\
6 & 28  & 0.9643 & 0.8893 & -7.50\% & 0.8929 & \textbf{-7.14\%} & 0.8929 & \textbf{-7.14\%} & 0.8929 & \textbf{-7.14\%} \\
7 & 14  & 0.9286 & 0.9571 & 2.85\%  & 0.9429 & \textbf{1.43\%}  & 0.9786 & 5.00\%           & 0.9929 & \textbf{6.43\%} 
\\ \bottomrule
\end{tabular}
}
\end{table*}

\newpage
\noindent
\textbf{Visualisation Results} 
Figure~\ref{fig:vis_tsne_all} shows t-SNE visualization on all methods (SHARP and baselines). SHARP displays the clearest class boundaries across all datasets.

\begin{figure*}[http]
    \centering
    \begin{subfigure}{\linewidth}
        \includegraphics[width=1\linewidth]{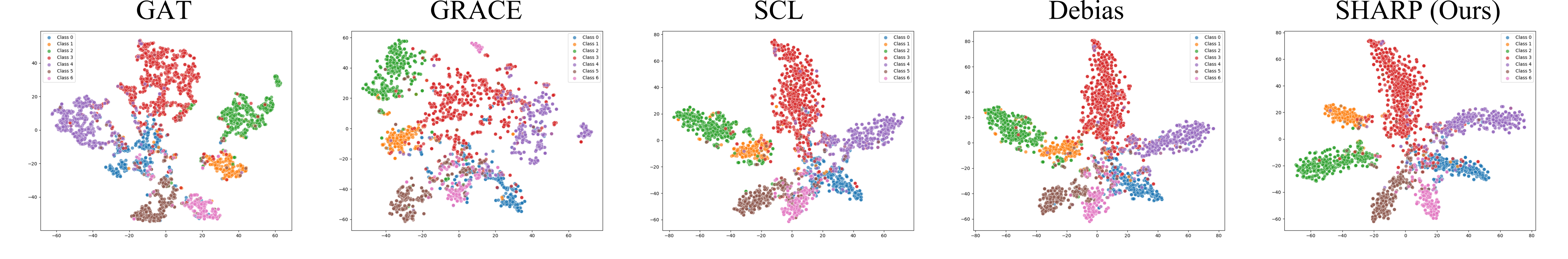} 
        \caption{Cora dataset}
    \end{subfigure}
    \begin{subfigure}{\linewidth}
        \includegraphics[width=0.99\linewidth]{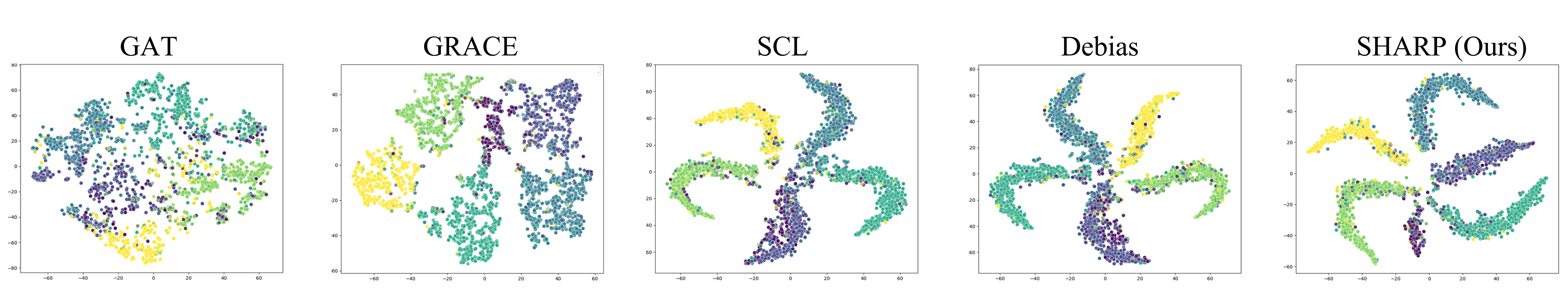} 
        \caption{Citeseer dataset}
    \end{subfigure}
    \begin{subfigure}{\linewidth}
        \includegraphics[width=1\linewidth]{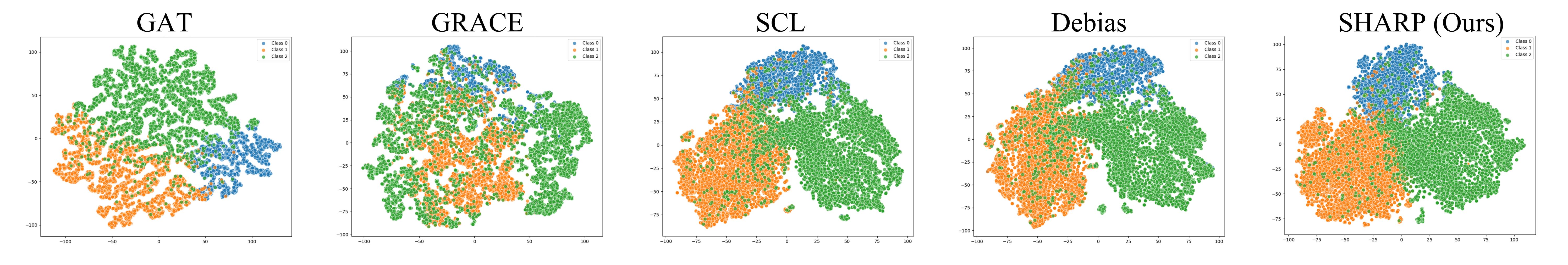} 
        \caption{PubMed dataset}
    \end{subfigure}
    \begin{subfigure}{\linewidth}
        \includegraphics[width=1\linewidth]{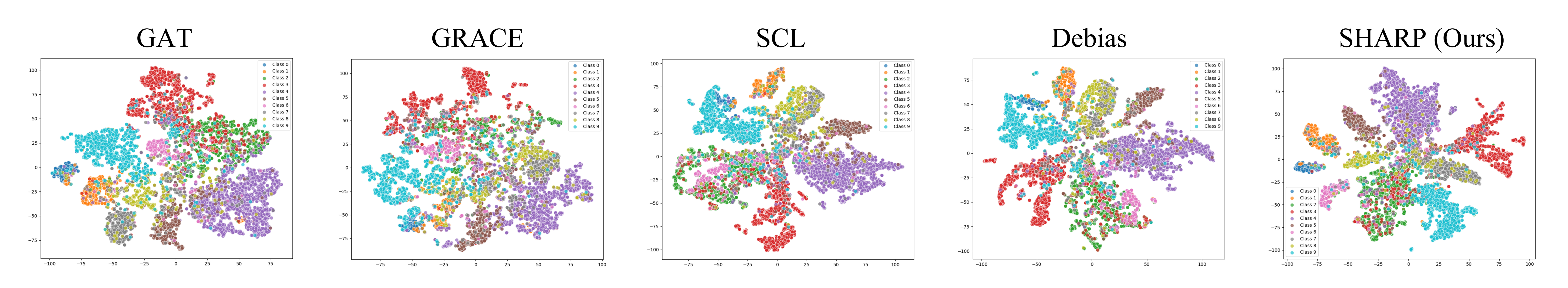} 
        \caption{Wiki-CS dataset}
    \end{subfigure}
    \caption{The t-SNE visualizations of embeddings on different models (GAT Layer, $r=0.3$)}
    \label{fig:vis_tsne_all}
\end{figure*}

\clearpage  

\onecolumn
\section{Theoretical Analysis of Worst-case Scenario}
\label{sec:theoretical}
This section presents an analysis based on the work of Robinson et al.~\cite{robinson2020contrastive}. In contrast to the approach described in~\cite{robinson2020contrastive}, which samples from an underlying distribution based on similarity in unsupervised learning, our method employs a sigmoid classifier $\sigma(z): \mathcal{Z}\times \mathcal{C} \rightarrow [0, 1]$ on representation space $\mathcal{Z}$. This classifier is trained on labelled data to predict the unlabelled data, with marginal distribution $p(x)$. The positive and negative samples are associated with the distributions $x^+ \sim q^{+}$ and $x^- \sim q^{-}$, respectively, and are related to the embedding function. With slight abuse of notations, we refer to the embedding function as $g$. Specifically:
\begin{align*}
    q^{+}_{\alpha}(x^+) &\propto (e^{g(x)g(x^+)} + \alpha e^{- g(x)g(x^+)}) \cdot \sigma_{\textrm{sigmoid}}(g_p(x^+)) \\ q^{-}_\beta(x^-) &\propto e^{\beta g(x)g(x^-)} \cdot \sigma_{\textrm{sigmoid}}(g_p(x^-))
\end{align*}
where $\alpha, \beta \in (0, 1]$, $g_p$ denotes the embedding function learned under $p$.
Increasing the scalar $\alpha$ results in reduced weights for positive samples (note that it is the normalised weights), while increasing $\beta$ leads to higher weights for negative samples in the underlying distribution. Given the symmetric property of the HAR loss function (Equations \eqref{eq:HAR_loss} and \eqref{eq:ell_loss}), the loss function can be expressed as the expectation of positive and negative samples from the transformation. We can reformulate the following loss function $\mathcal{L}$ to align with previous work~\cite{chuang2020debiased,robinson2020contrastive}:
\begin{align}\label{eq:large_loss}
    \mathcal{L}(g, \alpha, \beta)
    =\mathbb{E}_{x \sim \mathcal{T}, x^+ \sim p^+}
        \left[-\log 
            \frac{\textsc{I}_\alpha}{\textsc{I}_\alpha+\textsc{II}}
        \right],
\end{align}
where
\begin{align*}
    \textsc{I}_\alpha &= \mathbb{E}_{x^+ \in \{ x'^+, x''^+ \}} [e^{\bar{g}(x)\bar{g}(x^+)}] \\
    &\quad + \alpha \mathbb{E}_{x^+ \notin \{ x'^+, x''^+ \}}[e^{-\bar{g}(x)\bar{g}(x^+)}],\\
    \textsc{II} &= \mathbb{E}_{x^- \sim q_{\beta}^-}[e^{\beta\bar{g}(x)\bar{g}(x^-)}].
\end{align*}
Here, $\bar{g}(t)\coloneq g(t)/(\sqrt{\tau}\cdot\|g(t)\|)$ is the normalised embedding function, and we have the bounds $e^{-1/\tau} \leq |e^{\bar{g}(u)\bar{g}(v)}| \leq e^{1/\tau}$ for any $u, v \in \mathcal{X}$. $\{ x'^+, x''^+ \}$ refers to the set of the pairs under graph argumentation. The first part of $\textsc{I}_\alpha$ refers to the pairs of (two) data points after augmentation, while the second term is greater than 0.
Taking $\alpha=0$, the Equation~\eqref{eq:large_loss} degenerates to the case described in~\cite{robinson2020contrastive}.
Further, given fixed $\beta$ and by monotonicity property of logarithm function, we observe that the upper bound of loss function is reached when $\alpha=0$ (w.r.t. the ``lower'' classification performance), Therefore, we have the following proposition:

\begin{proposition}\label{prop:alpha}
    With a fixed $\beta$ and when $\alpha=0$, the worst-case performance of SHARP is bounded by the performance of an unsupervised method~\cite{robinson2020contrastive}. That is, the SHARP's worst-case performance is when SHARP degenerates to an unsupervised contrastive learning approach. 
\end{proposition}

Before analysing $\mathcal{L}(g, \beta) \coloneq \mathcal{L}(g, 0, \beta)$, one may note that the classifier is learned on a labelled dataset. Intuitively, the classifier can assign higher weights to negative samples than those obtained by directly sampling from the dataset. We have the following proposition:

\begin{proposition}[Sampling ratio]\label{prop:ratio}
    The probability of a selected sample being a negative sample based on our method is at least equal to that of the probability of a randomly selected sample being a negative sample from the marginal distribution in~\cite{robinson2020contrastive}. That is, the sampling ratio $w(x^-)$ is at least greater than or equal to 1, and exists a constant,
    \begin{align*}
        w(x^-) \coloneq \frac{\sigma_{\textrm{sigmoid}}(g_{p}(x^-))}{p(x^-)} \geq 1,
    \end{align*}
    and
    \begin{align*}
        \mathbb{E}_{x^-\sim p^{-}}[w(x^-)] = \mathrm{const}.
    \end{align*}
\end{proposition}

We have the following lemma on the worst-case objective function, building on the Proposition 3 in~\cite{robinson2020contrastive}.

\begin{lemma}[Worst-case negative sampling]
    When $\alpha=0$, let $\bar{\mathcal{L}}(g) = \sup_{\beta} \mathcal{L}(g, \beta)$. For any $\tau>0$, we have $\mathcal{L}(g, \beta)\rightarrow \bar{\mathcal{L}}(g)$ as $\beta\rightarrow\infty$.
\end{lemma}

\begin{proof}
Similar to the proof in~\cite{robinson2020contrastive}, we consider the essential supremum,
\begin{align*}
    M(x) = \sup \left\{ 0< m \leq 1: 
    \begin{array}{l}
        \bar{g}(x) \bar{g}(x^-)\leq m,\\
        \textrm{a.s. for}~ x^{-}\sim p^{-} 
    \end{array}
    \right\},
\end{align*}
with
\begin{align*}
    \bar{\mathcal{L}}(g, \beta)
    =\mathbb{E}_{x \sim \mathcal{T}, x^+ \sim p^+}
        \left[-\log 
            \frac{\textsc{I}_{\alpha=0}}{\textsc{I}_{\alpha=0} + \bar{Q} e^{M(x)}}
        \right],
\end{align*}
where $\bar{Q}$ is a constant of upper bound of the ratio of negative samples to positive samples. 
Recall that in a balanced and binary classification, we may have $\bar{Q}=1$. In cases of multi-class classification, $\bar{Q}$ is a constant chosen based on the ratio of the most imbalanced class for analysis.
We have the difference of the two terms,
\begin{align*}
    &\quad\left| \bar{\mathcal{L}}(g)-\mathcal{L}(g, \beta)\right|\\
    &= \mathbb{E}_{x \sim \mathcal{T}, x^+ \sim p^+} \left| \log \left( \bar{g}(x)\bar{g}(x^+)+\bar{Q}\mathbb{E}_{x^- \sim q_{\beta}^-}[\bar{g}(x)\bar{g}(x^-)] \right) \right. \\ 
    &\quad - \left. \log \left( \bar{g}(x)\bar{g}(x^+) + \bar{Q}e^{M(x)} \right) \right|\\
    &\leq \frac{e^{1/\tau}\bar{Q}}{\bar{Q}+1} \mathbb{E}_{x \sim \mathcal{T}}\mathbb{E}_{x^- \sim q_{\beta}^-}\left|e^{M(x)}-e^{\bar{g}(x)\bar{g}(x^-)}\right|\\
    &\leq e^{2/\tau} \mathbb{E}_{x \sim \mathcal{T}}\mathbb{E}_{x^- \sim q_{\beta}^-}\left|M(x)-\bar{g}(x)\bar{g}(x^-)\right|
\end{align*}

Following prior work, the first inequality is due to Lipschitz of logarithm function and rearranging the terms. Note that the expectation does not contain terms with $x^+$. The second inequality is due to $|e^{\bar{g}(x)\bar{g}(x^-)}| \leq e^{1/\tau}$.
Now we analyse the expectation of $|M(x)-\bar{g}(x)\bar{g}(x^-)|$. Given a fixed sample $x$, and a set of ``good'' events $\mathcal{G}_\epsilon=\{x^-:\bar{g}(x)\bar{g}(x^-) \geq M(x)-\epsilon\}$, and its complement set $\mathcal{G}_\epsilon^c$, for a fixed $t$ and $\epsilon>0$,
\begin{align*}
    &~\quad \mathbb{E}_{x^- \sim q_{\beta}^-}\left|M(x)-\bar{g}(x)\bar{g}(x^-)\right|\\
    &= \mathbb{P}_{x^- \sim q_{\beta}^-}(\mathcal{G}_\epsilon)\mathbb{E}_{x^- \sim q_{\beta}^-}\left[ \left|M(x)-\bar{g}(x)\bar{g}(x^-)\right| | \mathcal{G}_\epsilon \right] \\
    &\quad + \mathbb{P}_{x^- \sim q_{\beta}^-}(\mathcal{G}_\epsilon^c)\mathbb{E}_{x^- \sim q_{\beta}^-}\left[ \left|M(x)-\bar{g}(x)\bar{g}(x^-)\right| | \mathcal{G}_\epsilon^c \right]\\
    &\leq \epsilon + 2\mathbb{P}_{x^- \sim q_{\beta}^-}(\mathcal{G}_\epsilon^c)
\end{align*}

Further, with the Proposition~\ref{prop:ratio},
\begin{align*}
    &~\quad \mathbb{P}_{x^- \sim q_{\beta}^-}(\mathcal{G}_\epsilon^c)\\
    &= \int_{\mathcal{X}} \mathbb{I}_{(\bar{g}(x)\bar{g}(x^-) < M(x)-\epsilon)} \frac{e^{\beta g(x)g(x^-)} \cdot p(x^-)}{Z_\beta} \cdot w(x^-) \mathrm{d}x^-\\
    & \leq \frac{e^{\beta(M(x)-\epsilon)}}{Z_\beta} \mathbb{E}_{x^- \sim p^-} [w(x^-)]\\
    & \leq \frac{e^{\beta(M(x)-\epsilon)}}{Z_\beta} \cdot \textrm{const}
\end{align*}
where 
\begin{align*}
    Z_\beta
    &=\int_{\mathcal{X}}e^{\beta g(x)g(x^-)} \cdot \sigma_{\textrm{sigmoid}}(g_p(x^-)) \mathrm{d}x^- \\
    &=\int_{\mathcal{X}}e^{\beta g(x)g(x^-)} \cdot p(x^-) \cdot \frac{\sigma_{\textrm{sigmoid}}(g_p(x^-))}{p(x^-)} \mathrm{d}x^-\\
    &\geq e^{\beta(M-\epsilon/2)} \mathbb{P}_{x^- \sim p^-}(g(x)g(x^-)>M-\epsilon/2)
\end{align*}\

With $\rho_\epsilon \coloneq \mathbb{P}_{x^-\sim p^-}\left( (\bar{g}(x^-)\bar{g}(x^-)\leq M(x)-\epsilon/2) \right) > 0$, similarly, with $\beta \rightarrow \infty$ we have
\begin{align*}
    \mathbb{P}_{x^- \sim q_{\beta}^-}(\mathcal{G}_\epsilon^c) \leq \frac{e^{\beta(M(x)-\epsilon)}}{e^{\beta(M(x)-\epsilon/2)}} \cdot \rho_{\epsilon}^{-1} = e^{-\beta \epsilon/2} \rho_{\epsilon}^{-1} \rightarrow 0
\end{align*}
We have the term $\left| \bar{\mathcal{L}}(g)-\mathcal{L}(g, \beta)\right|\rightarrow 0$ with sufficient small $\epsilon$.
\end{proof}

This proposition also suggests that the theorems provided in \cite{robinson2020contrastive} hold. We provide a brief interpretation of the theorems for reference, the Lemma \ref{lemma:T4_robinson} and \ref{lemma:T5_robinson}.

\begin{lemma}[Theorem 4 in \cite{robinson2020contrastive}, informal]\label{lemma:T4_robinson}
The downstream classification task can be viewed as a packing problem. The solution of minimising the $\bar{\mathcal{L}}(g)$ are a maximum margin clustering in packing problem.

\end{lemma}

\begin{lemma}[Theorem 5 in \cite{robinson2020contrastive}, informal]\label{lemma:T5_robinson}
    Under Lemma~\ref{lemma:T4_robinson}, there exists a solution to the packing problem achieving bounded misclassification risk for the downstream classifier.
\end{lemma}

For the proof of the lemmas, we refer readers to see the Appendix A2 in \cite{robinson2020contrastive}.
Overall, it suggests that regarding the selection of hyper-parameters, if the classifier achieves adequate performance than random sampling, the performance is better than purely unsupervised case. 
Based on all the above analysis, we have the following corollary:

\begin{corollary}
    The misclassification risk of SHARP is bounded, and the objective function of SHARP is not higher than that of the unsupervised objective function in~\cite{robinson2020contrastive}.
\end{corollary}

\end{document}